\relax
\documentclass[a4paper,11pt]{article}
\usepackage[square, sort]{natbib}
\usepackage[nottoc,notlot,notlof]{tocbibind}
\usepackage{times}  
\usepackage{helvet} 
\usepackage{courier}  
\usepackage{url}
\usepackage{fullpage}
\usepackage{graphicx}
\usepackage{graphics}
\usepackage{xcolor}
\definecolor{darkblue}{rgb}{0,0.08,0.8}

\usepackage{subcaption}
\usepackage{lmodern}
\usepackage{datatool}
\usepackage{algorithm}
\usepackage[noend]{algpseudocode}

\usepackage{amsthm,amsmath,amssymb}
\usepackage{bm}
\usepackage{latexsym}
\usepackage{float}
\usepackage{color}
\usepackage{pifont}
\newtheorem{theorem}{Theorem}[section]
\newtheorem*{theorem*}{Theorem}

\newtheorem{lemma}[theorem]{Lemma}
\newtheorem{proposition}[theorem]{Proposition}
\newtheorem*{proposition*}{Proposition}

\renewcommand{\hat}{\widehat}

\newcommand{\E}{\mathbb{E}}

\usepackage[colorlinks=true,linkcolor=darkblue, citecolor=darkblue]{hyperref}

\usepackage{booktabs}       
\usepackage{nicefrac}       
\usepackage{microtype}      
\usepackage{multirow}
\usepackage{siunitx}

\usepackage{microtype}
\usepackage{mathtools}
\usepackage{adjustbox}
\usepackage{graphicx}
\usepackage[english]{babel}
\usepackage[utf8]{inputenc}
\usepackage{caption}

\newcommand{\norm}[1]{\left\lVert#1\right\rVert}

\newcolumntype{C}{>{$}c<{$}}
\newcolumntype{L}{>{$}l<{$}}

\begin{document}
\title{Model Imitation for Model-Based Reinforcement Learning}
\author{Yueh-Hua Wu$^{1,2}$\thanks{Equal contributions.}, Ting-Han Fan$^{3}$\footnotemark[1], Peter J. Ramadge$^{3}$, Hao Su$^{2}$}
\date{
$^1$ National Taiwan University\\ 
\vspace{0.03in}
$^2$ University of California San Diego\\
\vspace{0.05in}
$^3$ Princeton University
}
\maketitle

\begin{abstract}
Model-based reinforcement learning (MBRL) aims to learn a dynamic model to reduce the number of interactions with real-world environments. However, due to estimation error, rollouts in the learned model, especially those of long horizons, fail to match the ones in real-world environments. This mismatching has seriously impacted the sample complexity of MBRL. The phenomenon can be attributed to the fact that previous works employ supervised learning to learn the one-step transition models, which has inherent difficulty ensuring the matching of distributions from multi-step rollouts. Based on the claim, we propose to learn the transition model by matching the distributions of multi-step rollouts sampled from the transition model and the real ones via WGAN. We theoretically show that matching the two can minimize the difference of cumulative rewards between the real transition and the learned one. Our experiments also show that the proposed Model Imitation method can compete or outperform the state-of-the-art in terms of sample complexity and average return.
\end{abstract}

\section{Introduction}
Reinforcement learning (RL) is of great interest because many real-world problems 
can be modeled as sequential decision-making problems. When interactions with the environment are cheap, model-free reinforcement learning (MFRL) has been favored for both its simplicity and ability to learn complex tasks.
However, in most practical problems (e.g., autonomous driving), these interactions are extremely costly and this can make MFRL infeasible. 
An additional critique of MFRL is that it does not fully exploit past queries of the environment. This motivates the consideration of model-based reinforcement learning (MBRL). Along with learning an agent policy, MBRL learns a model of the dynamics of the environment with which the agent interacts.
If the model is sufficiently accurate, the agent can learn a desired policy by interacting 
with a simulation of the model. 
Hence by using samples from both the environment and from the model,
MBRL can reduce the number of environment samples needed to learn an acceptable policy.

Most previous work on MBRL employs supervised learning with an $\ell_2$-based error metric \citep{luo2018slbo,kurutach18metrpo,clavera2018mbmpo}, or maximum likelihood \citep{janner2019trust}, to learn a model. This is a nontrivial task in itself. Moreover, using the model to learn a policy incurs other challenges. Since the estimation error of the learned model accumulates as the trajectory grows, it is hard to train a policy on a long synthetic trajectory. Conversely, training on short trajectories makes the policy short-sighted. This issue is known as the planning horizon dilemma \citep{langlois2019benchmarking}. As a result, although based on a valuable insight, MBRL still has to be designed meticulously.

Intuitively, we would like to learn a transition model in a way that it can reproduce the trajectories that have been generated from the real dynamcis. However, supervised learning with one-step loss, as adopted by many MBRL methods, does not guarantee the resemblance of trajectories. 
Some previous work has proposed multi-step training \citep{luo2018slbo,asadi2019combating,talvitie2017self}. Still, experiments show that model learning fails to benefit much from the multi-step loss. We attribute this to the essence of supervised learning, which preserves the similarity of transitions only over a few steps, so the long-term objective of similarity between real and synthetic trajectories cannot be guaranteed. As we'll discuss in section~\ref{sec:matching}, supervised learning has quadratic error w.r.t. the trajectory length while our proposed approach incurs linear error.

We propose to learn a transition model via distribution matching. Specifically, we use a WGAN \citep{wgan} to match the distributions of state-action-next-state triple $(s,a,s')$ 
from the environment and the learned model so that the agent policy 
generates similar trajectories when interacting with either the environment or the learned model. Figure~\ref{fig:illustrate} illustrates the difference between methods based on supervised learning and distribution matching.
In detail, we sample trajectories from the real transition according to a policy. 
To learn the real transition $T$, 
we also sample synthetic trajectories from our learned model $T'$ with the same policy. The learned model serves as the generator in the WGAN framework and there is a critic that discriminates the two sources of trajectories. We update the generator and the critic alternatively until the synthetic trajectories cannot be distinguished from the real trajectories. This gives better theoretical guarantees than traditional supervised learning methods.
Compared to the ensemble methods in previous work, the proposed method is capable of generalizing to unseen transitions using only \emph{one} model of the dynamics. 
By contrast, using multiple models in an ensemble approach does not resolve the problem that one-step (or few-step) supervised learning fails to accurately imitate the distribution of rollouts. 

Our contributions are summarized below:
\begin{itemize}
\setlength\itemsep{0pt}

    \item We propose an MBRL method called model imitation (MI), in which the learned transition model generates similar rollouts to the environment, resulting in a more accurate policy gradient;
    \item We show that the learned transition is consistent: $T'\rightarrow T$ and the difference in cumulative rewards is linearly bounded by the training error of the WGAN;
    \item To stabilize model learning, we analyze our sampling technique and investigate training across WGANs;
    \item We experimentally show that MI is more sample efficient than state-of-the-art MBRL and MFRL methods.
\end{itemize}

\begin{figure}
    \centering
    \includegraphics[scale=0.4]{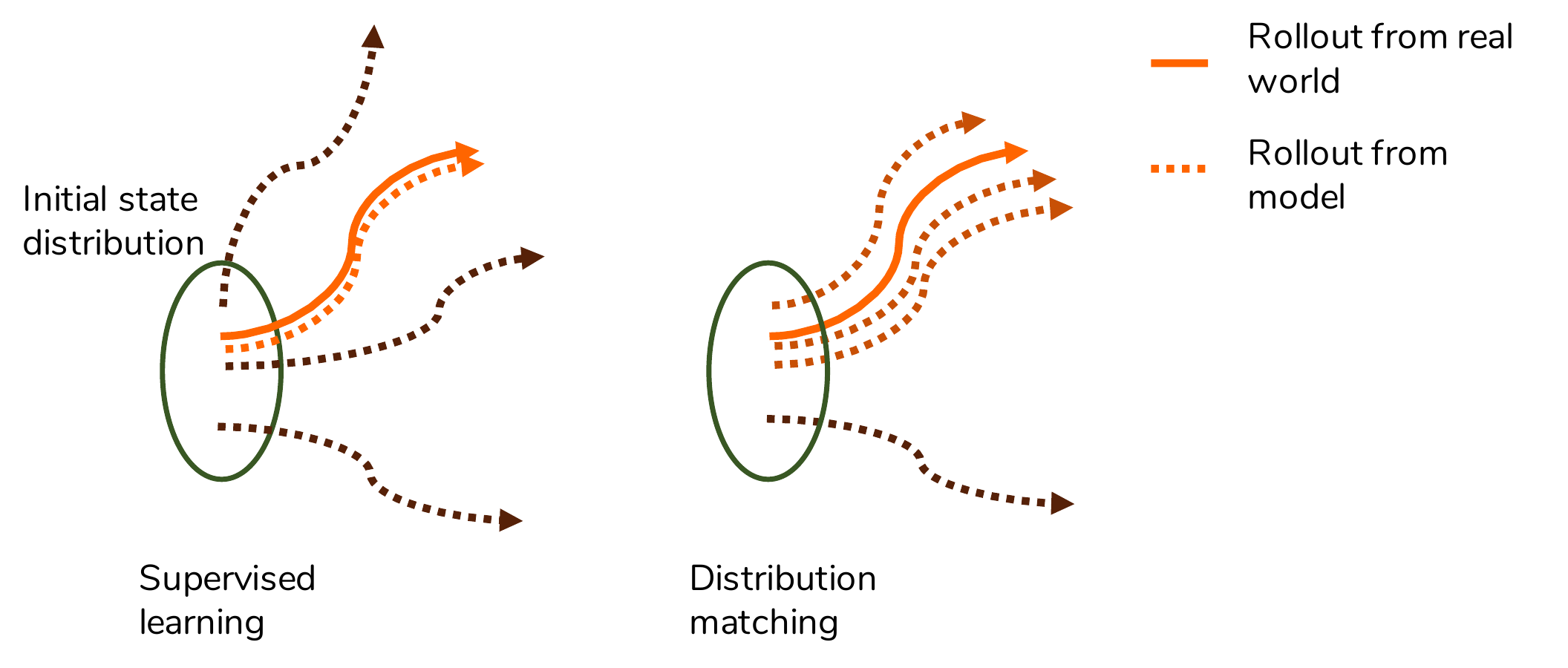}
    \caption{Distribution matching enables the learned transition to generate similar rollouts to the real ones even when the policy is stochastic or the initial states are close. In contrast, training with supervised learning does not ensure rollout similarity and the resulting policy gradient may be inaccurate. This figure considers a fixed policy sampling in the real transition and a transition model.}
    \label{fig:illustrate}
\end{figure}

\section{Related work}\label{relatedwork}
In this section, we introduce our motivation inspired by learning from demonstration (LfD) \citep{schaal1997learning} and give a brief survey of MBRL methods.
\subsection{Learning from Demonstration}
A straightforward approach to LfD is to leverage behavior cloning (BC), which reduces LfD to a supervised learning problem. Even though learning a policy via BC is time-efficient, it cannot imitate a policy without sufficient demonstration because the error may accumulate without the guidance of an expert \citep{ross2011dagger}. 
Generative Adversarial Imitation Learning (GAIL) \citep{ho2016generative} is another state-of-the-art LfD method that learns an optimal policy by utilizing generative adversarial training to match occupancy measures \citep{syed2008apprenticeship}.
GAIL learns an optimal policy by matching the distribution of the trajectories generated from an agent policy with the distribution of the given demonstration. 
Moreover, the two distributions match if and only if the agent has learned the optimal policy \cite{ho2016generative}. One advantage of GAIL is that it only requires a small amount of demonstration data to obtain an optimal policy. However, it requires a considerable number of interactions with 
the environment for the generative adversarial training to converge.

Our intuition is that transition learning (TL) is similar to learning from demonstration (LfD) by exchanging the roles of transition and policy.
In LfD, trajectories sampled from the real transition are given, and the goal is to learn a policy. While in TL, trajectories sampled from a fixed policy are given, and we want to learn the underlying transition of the environment. Thus from LfD to TL, we interchange the roles of the policy and the transition. It is therefore tempting to study the counterpart of GAIL in TL; i.e., learning the real transition by distribution matching. By doing so, we retain the advantages of GAIL, while its disadvantages are less critical because sampling from the learned model is much cheaper than sampling from the environment.
That GAIL learns a better policy than BC suggests that distribution matching has the potential to learn a better model of transition generation than supervised learning.

\subsection{Model-Based Reinforcement Learning}
For deterministic transitions, an $\ell_2$-based error is usually utilized to learn a transition model. 
In \citet{Nagabandi18}, an approach that uses supervised learning with a mean-squared error objective is shown to perform well under fine-tuning. 
To alleviate model bias, some previous work adopt ensembles \citep{kurutach18metrpo,jacob2018steve}, where multiple transition models with different initialization are trained at the same time. In a slightly more complicated manner, \citet{clavera2018mbmpo} utilizes meta-learning to gather information from multiple models. Lastly, SLBO \citep{luo2018slbo} is the first algorithm developed from solid theoretical properties for model-based deep RL via a joint model-policy optimization framework. 

For the stochastic transitions, maximum likelihood estimation 
or moment matching are natural ways to learn a transition model. 
This is usually modeled by a Gaussian distribution. Following this idea, 
Gaussian process models with and without model predictive control  \citep{Kamthe2017gauss,Kupcsik2013gauss,Deisenroth2015gauss} 
are studied as an uncertainty-aware version of MBRL. 
Similar to the deterministic case, to mitigate model bias and foster stability, an ensemble method for probabilistic networks \citep{kurtland2018pets} is also studied. 
An important distinction between training a deterministic transition model and a stochastic one is that the stochastic version can model noise from the real transition itself or from the insufficiency of data. This is another motivation given by \citet{kurtland2018pets}.

\section{Background}
\subsection{Reinforcement Learning}
We consider the standard Markov Decision Process (MDP) \citep{sutton1998introduction}. MDP is represented by a tuple $\langle S,\mathcal{A},T, r,\gamma\rangle$, where $\mathcal{S}$ is the state space, $\mathcal{A}$ is the action space, $T(s_{t+1}\vert s_t,a_t)$ is the transition density of state $s_{t+1}$ at time step $t+1$ given action $a_t$ made under state $s_t$, $r(s,a)$ is the reward function, and $\gamma\in(0,1)$ is the discount factor.

A stochastic policy $\pi(a\vert s)$ is a density of action~$a$ given state~$s$. Let the initial state distribution be $\alpha$. 
In the infinite horizon setting, the performance $R(\alpha,\pi,T)$ of the triple $(\alpha,\pi,T)$ is the expected cumulative $\gamma$-discounted reward. Equivalently, $R(\alpha,\pi,T)$ is the expected cumulative reward in a length-$H$ trajectory $\{s_t,a_t\}_{t=0}^{H-1}$ generated by $(\alpha,\pi,T)$ with $H\sim \text{Geometric}(1-\gamma)$. 
\begin{align}
\begin{split}\label{eq:cumulative_reward}
    R(\alpha,\pi,T)
    &=\textstyle \mathbb{E}\left[\sum_{t=0}^\infty \gamma^tr(s_t,a_t)\Big\lvert \alpha, \pi,T\right]\\
    &=\textstyle \mathbb{E}\left[\sum_{t=0}^{H-1} r(s_t,a_t)\Big\lvert \alpha, \pi,T\right].
\end{split}
\end{align}
When $\alpha$ and $T$ are fixed, $R(\cdot)$ depends solely on $\pi$, and reinforcement learning ~\citep{sutton1998introduction} aim to find a policy $\pi$ to maximize $R(\pi)$.

\subsection{Occupancy Measure}
Given initial state distribution $\alpha(s)$, policy $\pi(a|s)$ and transition $T(s'|s,a)$, the normalized occupancy measure $\rho_T^{\alpha,\pi}(s,a)$ 
generated by $(\alpha,\pi,T)$ is defined as
\begin{align}
\begin{split}
    \rho_T^{\alpha,\pi}(s,a) 
    &= \textstyle \sum\limits_{t=0}^{\infty} (1-\gamma)\gamma^t \mathbb{P}(s_t=s,a_t=a|\alpha,\pi,T)\\ 
    &\textstyle =(1-\gamma)\sum\limits_{t=0}^{H-1}  \mathbb{P}(s_t=s,a_t=a|\alpha,\pi,T),
    \label{eq:occupancy}
\end{split}
\end{align}
where $\mathbb{P}(\cdot)$ is the probability measure (a density function if $\mathcal{S}$ or $\mathcal{A}$ is continuous). Intuitively, $\rho_T^{\alpha,\pi}(s,a)$ is the distribution of $(s,a)$ in a length-$H$ trajectory $\{s_t,a_t\}_{t=0}^{H-1}$ with $H\sim \text{Geometric}(1-\gamma)$ following the laws of $(\alpha,\pi,T)$. From \cite{syed2008apprenticeship}, the relation between $\rho_T^{\alpha,\pi}$ and $(\alpha,\pi,T)$ is characterized by the Bellman flow constraint. Specifically, $x=\rho_T^{\alpha,\pi}$ as defined in Eq.~(\ref{eq:occupancy}) is the unique solution to:
\begin{align}
\begin{split} \label{eq:bellman}
	\textstyle x(s,a)
	&=\pi(a|s)\Big[(1-\gamma)\alpha(s)  \\
	&\qquad + \gamma \int x(s',a')T(s|s',a')ds'da'\Big],\\
	\textstyle x(s,a)
	&\geq 0.
\end{split}	
\end{align}
In addition, \citet[Theorem 2]{syed2008apprenticeship} shows that $\pi(a|s)$ and $\rho_T^{\alpha,\pi}(s,a)$ have a one-to-one correspondence with $\alpha(s)$ and $T(s'|s,a)$ fixed; i.e., $\pi(a|s)\triangleq\frac{\rho(s,a)}{\int \rho(s,a)da}$ is the only policy whose occupancy measure is $\rho$.

With the occupancy measure, the cumulative reward Eq.~(\ref{eq:cumulative_reward}) can be represented as 
\begin{equation}
    R(\alpha,\pi,T)=\E_{(s,a)\sim \rho^{\alpha,\pi}_T}[r(s,a)]/(1-\gamma).
    \label{eq:occu_cum_reward}
\end{equation}
The goal of maximizing the cumulative reward can then be achieved by adjusting $\rho^{\alpha,\pi}_T$. This motivates us to adopt distribution matching approaches like WGAN \citep{wgan} to learn a transition model.

\section{Theoretical Analysis for WGAN}\label{sec:matching}
In this section, we present a consistency result and error bounds for WGAN \citep{wgan}. In the setting of MBRL, the training objective for WGAN is
\begin{align}
\begin{split}
    \underset{T'}{\min}~\underset{\norm{f}_{\text{Lip}}\leq 1}{\max}~&
    \E_{(s,a)\sim \rho_T^{\alpha,\pi},~s'\sim T(\cdot|s,a)}[f(s,a,s')]-\\
    &\E_{(s,a)\sim \rho_{T'}^{\alpha,\pi},~s'\sim T'(\cdot|s,a)}[f(s,a,s')].
\label{eq:wgan}
\end{split}
\end{align}
By Kantorovich-Rubinstein duality \citep{Villani2008_opt_transport}, the optimal value of the inner maximization is exactly the 1-Wasserstein distance $W_1(p(s,a,s')\lvert\rvert p'(s,a,s'))$ where $p(s,a,s')=\rho_T^{\alpha,\pi}(s,a)T(s'|s,a)$ is the $\gamma$-discounted distribution of $(s,a,s')$. Thus, by minimizing over the choice of $T'$, we are finding $p'$ that minimizes $W_1(p(s,a,s')\lvert\rvert p'(s,a,s'))$, which gives the consistency.

\begin{proposition}[Consistency for WGAN]\label{prop:consistent-wgan} 
Let $T$ and $T'$ be the real and 
learned transitions respectively. 
If WGAN is trained to optimality, then
 $$T(s'|s,a)=T'(s'|s,a),~\forall (s,a)\in\text{Supp}(\rho_{T}^{\alpha,\pi}),$$
where $\text{Supp}(\rho_{T}^{\alpha,\pi})$ is the support of $\rho_{T}^{\alpha,\pi}$.
\end{proposition}

\begin{proof}
Because the loss function of WGAN is the 1-Wasserstein distance, we know $p(s,a,s')=p'(s,a,s')~\text{a.e.}$ at optimality. 
The Bellman equation (\ref{eq:bellman}) implies
	\begin{align*}
	&\rho_{T'}^{\alpha,\pi}(s,a) =\\
    	&\pi(a|s)\Big[(1-\gamma)\alpha(s)
		 + \gamma \int_{s',a'} \rho_{T'}^{\alpha,\pi}(s',a')T'(s|s',a')\Big]\\ 
	&= \pi(a|s)\Big[(1-\gamma)\alpha(s) + \gamma \int_{s'a'} p'(s',a',s)\Big]\\
	&\overset{p=p'}{=} \pi(a|s)\Big[(1-\gamma)\alpha(s) 
		+ \gamma \int_{s'a'} p(s',a',s)\Big]\\
	&=\rho_{T}^{\alpha,\pi}(s,a).
	\end{align*}
	Recall $p(s,a,s')\triangleq\rho_{T}^{\alpha,\pi}(s,a)T(s'|s,a)$ and 
	$p'(s,a,s')\triangleq\rho_{T'}^{\alpha,\pi}(s,a)T'(s'|s,a).$ 
	Hence, that WGAN is optimal implies $T(s'|s,a)=T'(s'|s,a),~\forall (s,a)\in\text{Supp}(\rho_{T}^{\alpha,\pi}).$
\end{proof}
The support constraint is inevitable because the training data is sampled from $\rho_T^{\alpha,\pi}$ and guaranteeing anything beyond it is difficult. In our experiments, the support constraint is not an issue. The performance rises in the beginning, suggesting that $\text{Supp}(\rho_{T}^{\alpha,\pi})$ is initially large enough.

Now that training with WGAN gives a consistent estimate of the real transition, it is justifiable to use it to learn a transition model. Note, however, that the consistency result only discusses the optimal case. The next step is to analyze the non-optimal situation and observe how the cumulative reward deviates w.r.t. the training error.

\begin{theorem}[Error Bound for WGAN]
	Let $\rho_T^{\alpha,\pi}(s,a),~\rho_{T'}^{\alpha,\pi}(s,a)$ be the normalized occupancy measures generated by the real and learned transitions $T$, $T'$, respectively. 
	If the reward function is $L_r$-Lipschitz and the training error of WGAN is 
	$\epsilon$, then
	$|R(\pi,T)- R(\pi,T')| \leq \epsilon L_r/(1-\gamma)$.
	\label{thm:err-wgan}
\end{theorem}
\begin{proof}
	Observe that the occupancy measure $\rho_T^{\alpha,\pi}(s,a)$ is a marginal distribution of $p(s,a,s')=\rho_T^{\alpha,\pi}(s,a)T(s'|s,a)$. Because the distance between the marginal is upper bounded by that of the joint, we have
	\begin{align*}
	 W_1(\rho_T^{\alpha,\pi}(s,a)\lvert\rvert \rho_{T'}^{\alpha,\pi}(s,a))  &\leq W_1(p(s,a,s')\lvert\rvert p'(s,a,s'))\\&= \epsilon,
	\end{align*}
	where $W_1$ is the 1-Wasserstein distance. Then, the cumulative reward is bounded by
	\begin{equation*}
	\begin{split}
	&R(\pi,T)
	=\frac{1}{1-\gamma}\int_{s,a} r(s,a)\rho_T^{\alpha,\pi}(s,a)\\
	=&R(\pi,T')+\frac{1}{1-\gamma}\int_{s,a} r(s,a)\big(\rho_T^{\alpha,\pi}-\rho_{T'}^{\alpha,\pi}\big)(s,a)\\
	=&R(\pi,T')+\frac{L_r}{1-\gamma}\int_{s,a} \frac{r(s,a)}{L_r}\big(\rho_T^{\alpha,\pi}-\rho_{T'}^{\alpha,\pi}\big)(s,a)
\end{split}
\end{equation*}
\begin{equation*}
\begin{split}
	\leq& R(\pi,T')+\frac{L_r}{1-\gamma}\underset{\norm{f}_{\text{Lip}}\leq 1}{\sup}\int_{s,a} f(s,a)\big(\rho_T^{\alpha,\pi}-\rho_{T'}^{\alpha,\pi}\big)(s,a)\\
	=&R(\pi,T')+\frac{L_r}{1-\gamma}\underset{\norm{f}_{\text{Lip}}\leq 1}{\sup}\E_{ \rho_T^{\pi,\alpha}}[f(s,a)]-\E_{ \rho_{T'}^{\pi,\alpha}}[f(s,a)]\\
	=&R(\pi,T')+\frac{L_r}{1-\gamma}W_1(\rho_T^{\pi,\alpha}\lvert\rvert\rho_{T'}^{\pi,\alpha})\\
	\leq &R(\pi,T')+\epsilon\frac{L_r}{1-\gamma}.
	\end{split}
	\end{equation*}
The first inequality holds because $r(s,a)/L_r$ is 1-Lipschitz and the last equality follows from Kantorovich-Rubinstein duality \cite{Villani2008_opt_transport}. Since $W_1$ distance is symmetric, the same conclusion holds if we interchange $T$ and $T'$. Hence
	$$|R(\pi,T)-R(\pi,T')|\leq \epsilon L_r/(1-\gamma).$$
\end{proof}
Theorem~\ref{thm:err-wgan} indicates that if WGAN is trained properly; i.e., $\epsilon$ is small, the cumulative reward on a synthetic trajectory will be close to that on 
the real trajectory. As MBRL aims to train a policy on synthetic trajectories, 
the accuracy of the cumulative reward over synthetic trajectories is the \emph{bottleneck}. Theorem~\ref{thm:err-wgan} implies that WGAN's error is linear to the (expected) length of the trajectory $(1-\gamma)^{-1}$. This is a sharp contrast to the error bounds in most RL literature, where the dependency on the trajectory length is usually quadratic \citep{syed2010,ross2011dagger}. 
This is due to the error propagation: if an error $\epsilon$ is made at each step, then the total error is $\epsilon+2\epsilon+...+H\epsilon=O(\epsilon H^2)$. Since WGAN gives us a better estimation of the cumulative reward in the learned model, the policy update becomes more accurate.

\section{Model Imitation for Model-Based Reinforcement Learning}
\label{sec:model-imitation}
In this section, we present a practical MBRL method called model imitation (MI) that incorporates the transition learning mentioned in Section~\ref{sec:matching}.

\subsection{Sampling Technique for Transition Learning}
It is hard to train the WGAN using a long synthetic trajectory, since the trajectory gradually digresses from the real behavior over time. To tackle this issue, we use the learned transition $T'$ to synthesize $N$ short trajectories with initial states sampled from the real trajectory.

To analyze this sampling technique, let $\beta<\gamma$ be the discount factor of the short trajectories so that the expected length is $\E[L]=(1-\beta)^{-1}$. To simplify the notation, let $\rho_{T'}^\beta$, $\hat{\rho}_T^\beta$, $\rho_T^\beta$, $\rho_T$ be the normalized occupancy measures of synthetic short trajectories, empirical real short trajectories, real short trajectories and real long trajectories. Notice both $\rho_{T'}^\beta$ and $\rho_{T}^\beta$ are generated from the same initial distribution $\rho_T$. The 1-Wasserstein distance is bounded by
$$
W_1(\rho_{T'}^\beta\lvert\rvert \rho_T)\leq W_1(\rho_{T'}^\beta\lvert\rvert \hat{\rho}_T^\beta) + W_1(\hat{\rho}_T^\beta\lvert\rvert \rho_T^\beta) + W_1(\rho_T^\beta \lvert\rvert \rho_T).
$$
$W_1(\rho_{T'}^\beta\lvert\rvert \hat{\rho}_T^\beta)$ is upper bounded by the training error of WGAN on short trajectories, which can be small empirically because the short ones are easier to imitate. By \cite{Canas2012wdist_bound}, $$W_1(\hat{\rho}_T^\beta\lvert\rvert \rho_T^\beta)=O(1/N^{1/d}),$$ where $d$ is the dimension of $(s,a)$. Also, by Lemma~2 and $W_1\leq D_{TV}\text{diam}(\mathcal{S\times A})$ \citep{dist_bounds}, we obtain $$W_1(\rho_T^\beta \lvert\rvert \rho_T)\leq \text{diam}(\mathcal{S\times A})\frac{(1-\gamma)\beta}{\gamma-\beta},$$ where $\text{diam}(\cdot)$ is set diameter. The second term encourages $N$ to be large while the third term wants $\beta$ to be small. Besides, $\beta$ cannot be too small otherwise there is little exploration to the future; in practice we may sample $N$ short trajectories to reduce the error from $W_1(\rho_{T'}\lvert\rvert \rho_T)$ to $W_1(\rho_{T'}^\beta\lvert\rvert \rho_T)$. Finally, since $\rho_{T'}^\beta$ is the occupancy measure we train on, from the proof of Theorem~\ref{thm:err-wgan} we deduce that 
$$|R(\pi,T)-R(\pi,T')|\leq W_1(\rho_{T'}^\beta\lvert\rvert \rho_T)L_r/(1-\gamma).$$ 
Thus, the WGAN does better under this sampling technique.

\subsection{Empirical Transition Learning}\label{sec:empirical}
To learn a transition model based on the occupancy measure matching mentioned in Section~\ref{sec:matching}, we align the distribution of the triples $(s,a,s')$ between the real and the learned transition models. This is inspired by how GAIL~\citep{ho2016generative} learns to align distributions over $(s,a)$ via solving an MDP with rewards extracted from a discriminator. 

To formulate a similar problem, recall from Eq.~(\ref{eq:wgan}), a 1-Lipschitz function $f(s,a,s')$, which we call the WGAN critic, is optimized in the inner maximization of the WGAN objective. Observe that $f$ is associated with the real transition $T$ while $-f$ is with the learned transition $T'$. To do the maximization, $f(s,a,s')$ tends to be positive if $(s,a,s')$ is drawn from $T$ and negative if drawn from $T'$. Thus, $f$ acts like a reward that signals the tendency of $(s,a,s')$ being generated from the real transition. Following this observation, consider an MDP $\langle\mathcal{S},\mathcal{A},T',f,\gamma\rangle$ with (psuedo) rewards being the WGAN critic $f(s,a,s')$. The overall goal of this MDP is to train $T'$ to maximize the cumulative pseudo reward under fixed policies. To see that this is equivalent to the WGAN objective, recall from \cite{wgan} that Eq.~(\ref{eq:wgan}) is typically trained by variants of gradient descent-ascent on $T'$ and $f$. However, the first expectation of Eq.~(\ref{eq:wgan}), $\E_{\pi,T}$, is independent of $T'$ under fixed $f$, so we may ignore it when optimizing $T'$. The second expectation $\E_{\pi,T'}$ is  the (normalized) cumulative pseudo reward
\begin{align}
\begin{split}
&\E_{(s,a)\sim \rho_{T'}^{\alpha,\pi},~s'\sim T'(\cdot|s,a)}[f(s,a,s')] \\
=&\E\Big[\sum_{t=0}^\infty \gamma^t f(s_t,a_t,s_{t+1})\Big| \alpha,\pi,T'\Big](1-\gamma).
\end{split}
\end{align}
This shows the equivalence between optimizing the cumulative pesudo reward and the WGAN objective. Later, we will introduce a transition learning version of PPO \citep{schulman2017proximal} to optimize the cumulative pseudo reward. 

Now consider the learning of WGAN critic $f(s,a,s').$ Since the policy and the learned model are updated alternately, we are required to train a sequence of $f$'s to account for the change of policies. However, the WGAN critic may be so strong that the generator (the learned transition) gets stuck at a local optimum \citep{zhang2018self}. The reason is that unlike a GAN that mimics the Jensen-Shannon divergence and hence its inner maximization is upper bounded by $\log(2)$, the WGAN mimics the Wasserstein distance and the inner maximization is unbounded from above. Such unboundedness implies the hardest instances dominate the attention of the generator, so the generator is likely to get stuck. As suggested in \citet{zhang2018self}, one should guide the generator to learn the easy task first and then gradually switch to the harder ones. A way to achieve it is to enforce boundness: introduce 
a parameter $\delta >0$ and cut-off the WGAN objective so that the inner maximization is upper bounded by $2\delta$:
\begin{align}
\begin{split}\label{eq:trun-wgan}
    \underset{T'}{\min}~\underset{\norm{f}_{\text{Lip}}\leq 1}{\max}~&
    \E_{\substack{(s,a)\sim \rho_T^{\alpha,\pi}\\s'\sim T(\cdot|s,a)}}[\min(\delta,f(s,a,s')]+\\
    &\E_{\substack{(s,a)\sim \rho_{T'}^{\alpha,\pi}\\s'\sim T'(\cdot|s,a)}}[\min(\delta,-f(s,a,s'))].
\end{split}
\end{align}
As $\delta\rightarrow\infty$, Eq.~(\ref{eq:trun-wgan}) recovers the WGAN objective Eq.~(\ref{eq:wgan}). To better comprehend Eq.~(\ref{eq:trun-wgan}), notice that it is equivalent to
\begin{equation}
\begin{split}
    \underset{T'}{\min}~\underset{\norm{f}_{\text{Lip}}\leq 1}{\max}~
    &\E_{\substack{(s,a)\sim \rho_T^{\alpha,\pi}\\s'\sim T(\cdot|s,a)}}[\min(0,f(s,a,s')-\delta)]+\\
    &\E_{\substack{(s,a)\sim \rho_{T'}^{\alpha,\pi}\\s'\sim T'(\cdot|s,a)}}[\min(0,-f(s,a,s')-\delta)],
    \end{split}
    \end{equation}
which is in turn equivalent to    
    \begin{equation}
    \begin{split}
    \underset{T'}{\min}~\underset{\norm{f}_{\text{Lip}}\leq 1}{\min}~
    &\E_{\substack{(s,a)\sim \rho_T^{\alpha,\pi}\\s'\sim T(\cdot|s,a)}}[\max(0,\delta-f(s,a,s'))]+\\
    &\E_{\substack{(s,a)\sim \rho_{T'}^{\alpha,\pi}\\s'\sim T'(\cdot|s,a)}}[\max(0,\delta+f(s,a,s'))].
\end{split}
\label{eq:hinge-wgan}
\end{equation}
This is a hinge loss version of the generative adversarial objective. We also empirically found that $\delta=1$ is enough for our purpose. This form of WGAN is first introduced in \cite{lim2017geometric}, where a consistency result is provided. According to \cite{lim2017geometric}, the inner minimization is interpreted as the soft-margin SVM. Consequently, it is margin-maximizing, which potentially enhances robustness. Further evaluation experiments are provided in \citet{zhang2018self,miyato2018spectral}.

\begin{algorithm}
\caption{Model Imitation for Model-Based Reinforcement Learning}
\begin{algorithmic}[1]\label{algo:mi}
\State Parameterize policy $\pi$, $T$, and WGAN critic $f$ with $\theta$, $\phi$, and $w$ respectively. Initialize an empty environment dataset $\mathcal{D}_\text{env}$
\State Initialize an empty queue $\mathcal{Q}_\pi$ of length $q$
\For {$i=0,1,2,...$}
\State Take actions in real environment according to $\pi_\theta$ and collect transition tuples $\mathcal{D}_\mathrm{i}$
\State Push $\pi_\theta$ into $\mathcal{Q}_\pi$
    \For{$N$ epochs}
    \For{$n_\text{transition}$ epochs}
        \State optimize Eq.~(\ref{eq:tloss}) and (\ref{eq:hinge-wgan-kai}) over $T_\phi$ and $f_w$ with $\{\mathcal{D}_j\}_{j=i-q+1}^i$ and $\mathcal{Q}_\pi$
    \EndFor
    \For{$n_\text{policy}$ epochs}
        \State update $\pi_\theta$ by TRPO on the data generated by $T_\phi$
    \EndFor
\EndFor
\EndFor
\end{algorithmic}
\end{algorithm}

After the above modification, the learned transition is modeled by a Gaussian distribution $T'(s'|s,a)=T_\phi(s'\vert s,a)\sim\mathcal{N}(\mu_\phi(s,a),\Sigma_\phi(s,a))$. 
This allows us to include both the stochastic and (approximately) deterministic scenarios. Although the underlying transitions of tasks like MuJoCo \citep{todorov2012mujoco} are deterministic, a Gaussian model enables the transition function to explore more during optimization. This improves the quality of the learned transition. Even through the transition is modeled as a stochastic function, when we optimize the policy, we use
the mean $\mu_\phi$ to estimate the next state. It is because learning from a stochastic function usually requires more samples and is more time-consuming. We may regard the learning from the mean of a Gaussian-parameterized transition model as the learning from an ensemble of the samples from a stochastic transition model.

Similar to previous work such as \citet{luo2018slbo}, the policy used to collect real transition tuples should possess a certain degree of exploration so as to derive a more generalized model. In \citet{luo2018slbo}, Ornstein-Uhlunbeck noise is added to the original policy to aid gathering a wider range of transition tuples. However, noise signals may greatly increase the complexity of occupancy measure matching.
Hence the scale of such noise needs to be carefully chosen using empirical methods. Instead of adding noise to enhance exploration, we enforce a lower bound $\sigma_\mathrm{min}$ of the standard deviation of the policy. This is found to be more stable for distribution matching.

Recall that the learned transition is trained using an MDP whose reward function is the critic of the truncated WGAN. To maximize the cumulative pseudo reward, we employ PPO \citep{schulman2017proximal}, which is an efficient approximation of TRPO \citep{schulman2015trust}. Although the PPO is originally designed for policy optimization, it can be adapted to transition learning with a fixed sampling policy and the PPO objective Eq.~(7) of \cite{schulman2017proximal}:
\begin{equation}
    \mathcal{L}_\text{PPO}(\phi)=\hat{\E}_t\Big[ \min(r_t(\phi)\hat{A}_t,~\text{clip}(r_t(\phi),1-\epsilon,1+\epsilon)\hat{A}_t) \Big],
    \label{eq:obj_ppo}
\end{equation}
where 
$$
r_t(\phi)=\frac{T_\phi(s_{t+1}|s_t,a_t)}{T_{\phi_{\text{old}}}(s_{t+1}|s_t,a_t)}
$$ 
and $\hat{A}_t$ is an advantage function derived from the pseudo reward $f(s_t,a_t,s_{t+1})$. Notice the pseudo reward is not truncated here because \citet{zhang2018self} suggests introducing the truncation only in the discriminator loss. That is, use $f$ in Eq.~(\ref{eq:obj_ppo}) and a hinge version of $f$ in Eq.~(\ref{eq:hinge-wgan-kai}).

To enhance stability, we add $\ell^2$ loss as a regularization term to the PPO objective. We empirically observe that this regularization yields better transition models, allowing the policy to improve more steadily. The overall loss of the transition learning becomes
\begin{align}\label{eq:tloss}
    \mathcal{L}_\text{transition}=-\mathcal{L}_\text{PPO}+\eta\mathcal{L}_{\ell^2},
\end{align}
where $\eta\geq 0$ and $\mathcal{L}_{\ell^2}$ is the $\ell^2$ loss that we use to measure the distance between  $\mu_\phi(s_t,a_t)$ and the ground truth $s_{t+1}$.  
It should be noted that the policy gradient used to optimize the transition function is not limited to PPO. Other policy gradient methods such as TD3 \cite{fujimoto2018td3} and TRPO \cite{schulman2015trust} can be directly applied without much modification.

\subsection{Training Across WGANs}
A standard model-based reinforcement learning method can be decomposed into the following steps: 
(1) At each generation $i$, sample real transition tuples $\mathcal{D}_i=\{s, a, s'\}$ with policy $\pi_i$. 
(2) Use the collected tuples to learn a transition model $T'$. 
(3) Use the trajectories synthesized by $T'$ to update the policy from $\pi_i$ to $\pi_{i+1}.$
These three steps are repeated until a desired policy is obtained. The procedure suggests that the learned transition $T'$ is refined by training across WGANs, which is similar to GAN adaptation~\cite{wang2018transferring,wu2018memory,wang2019minegan}, whose goal is to leverage the knowledge learned from one generation to another. As a result, we do not need to learn from scratch and the training can be more sample-efficient and time-saving.
Our method can be considered as a special case of GAN adaptation where the learned transition $T'$ is the generator and the optimal generators (the real transition) are the same over all generations.

\citet{wang2018transferring} reported that directly transferring network weights of the generator and the discriminator from one task to another enables the generator to produce images of higher quality in the target domain.
However, we empirically found that training across WGANs can still be unstable despite having the network weights initialized accordingly. We attribute this issue to the fact that the difference between $\rho_{T}^{\alpha,\pi_i}$ and $\rho_{T}^{\alpha,\pi_{i+1}}$ can be large. In the domain adaptation literature \cite{blitzer2006domain,daume2009frustratingly}, the source and the target domains is usually assumed to be close so that the correspondence can be possibly learned.
In order to ensure the desired proximity, we maintain a queue of length $q$ that collects policies $\pi_{i-q+1}, \pi_{i-q+2},\dots,\pi_{i}$.
The idea is that instead of directly matching occupancy measures $\rho_{T}^{\alpha,\pi_i}$ and $\rho_{T'}^{\alpha,\pi_i}$ by optimizing 
Eq.~(\ref{eq:hinge-wgan}), we match the occupancy measure mixtures obtained from the policies in the queue. Explicitly, the inner loop of Eq.~(\ref{eq:hinge-wgan}) 
is rewritten as
\begin{align}
\begin{split}
    \underset{\norm{f}_{\text{Lip}}\leq 1}{\min}~&
    \E_{\substack{(s,a)\sim \rho_T^{\alpha,\pi^{(q)}}\\s'\sim T(\cdot|s,a)}}[\max(0,\delta-f(s,a,s'))]+\\
    &\E_{\substack{(s,a)\sim \rho_{T'}^{\alpha,\pi^{(q)}}\\s'\sim T'(\cdot|s,a)}}[\max(0,\delta+f(s,a,s'))],
    \label{eq:hinge-wgan-kai}
  \end{split}
\end{align}
where $\rho_T^{\alpha,\pi^{(q)}}$ and $\rho_{T'}^{\alpha,\pi^{(q)}}$ are the mixed measures $\frac{1}{q}\sum_{j=0}^{q-1}\rho_{T}^{\alpha,\pi_{i-j}}$ and $\frac{1}{q}\sum_{j=0}^{q-1}\rho_{T'}^{\alpha,\pi_{i-j}}$, respectively.

Hence for a larger $q$, the source and the target domains become closer and training across WGANs is more likely to benefit from transferring network weights. Since we can save the empirical measures $\hat{\rho}_{T}^{\alpha,\pi_{i-j}}$ 
by preserving data $\mathcal{D}_{i-j}$, this method does not require additional samples from the real environment. Although it may take more computation to estimate $\rho_{T'}^{\alpha,\pi^{(q)}}$, we found that $q=2$ is sufficient to stabilize the WGAN transfer process.

Most MBRL methods have only one pair of model-policy update for each real environment sampling. We consider a training procedure similar to SLBO \citep{luo2018slbo}, which recognizes that the value function is dependent on the varying transition model. As a result, we have multiple update pairs for each real environment sampling. 
Our resulting \emph{model imitation (MI)} method is summarized in Algorithm~\ref{algo:mi}.

Note that unlike prior GAN works and its applications \cite{yu2017seqgan,mao2017least,karras2019style} which showed that the longer generative adversarial training produces a generator of higher quality, we empirically found that the occupancy measure matching is not required to converge in order to obtain a better transition model. 
By learning the process of fooling the discriminator, it is sufficient to prevent the transition model from overfitting the $\ell^2$ model loss and to learn to reproduce similar trajectories given the equivalent policy.

\section{Experiments}
In this section, we would like to answer the following questions about the proposed \emph{model imitation} algorithm.
(1) Does the algorithm outperform the state-of-the-art in terms of sample complexity and average return?
(2) Does the proposed algorithm benefit from distribution matching in the sense that it is superior to its model-free and model-based counterparts, TRPO and SLBO?

\begin{figure*}[]
    \centering
    \includegraphics[scale=0.31]{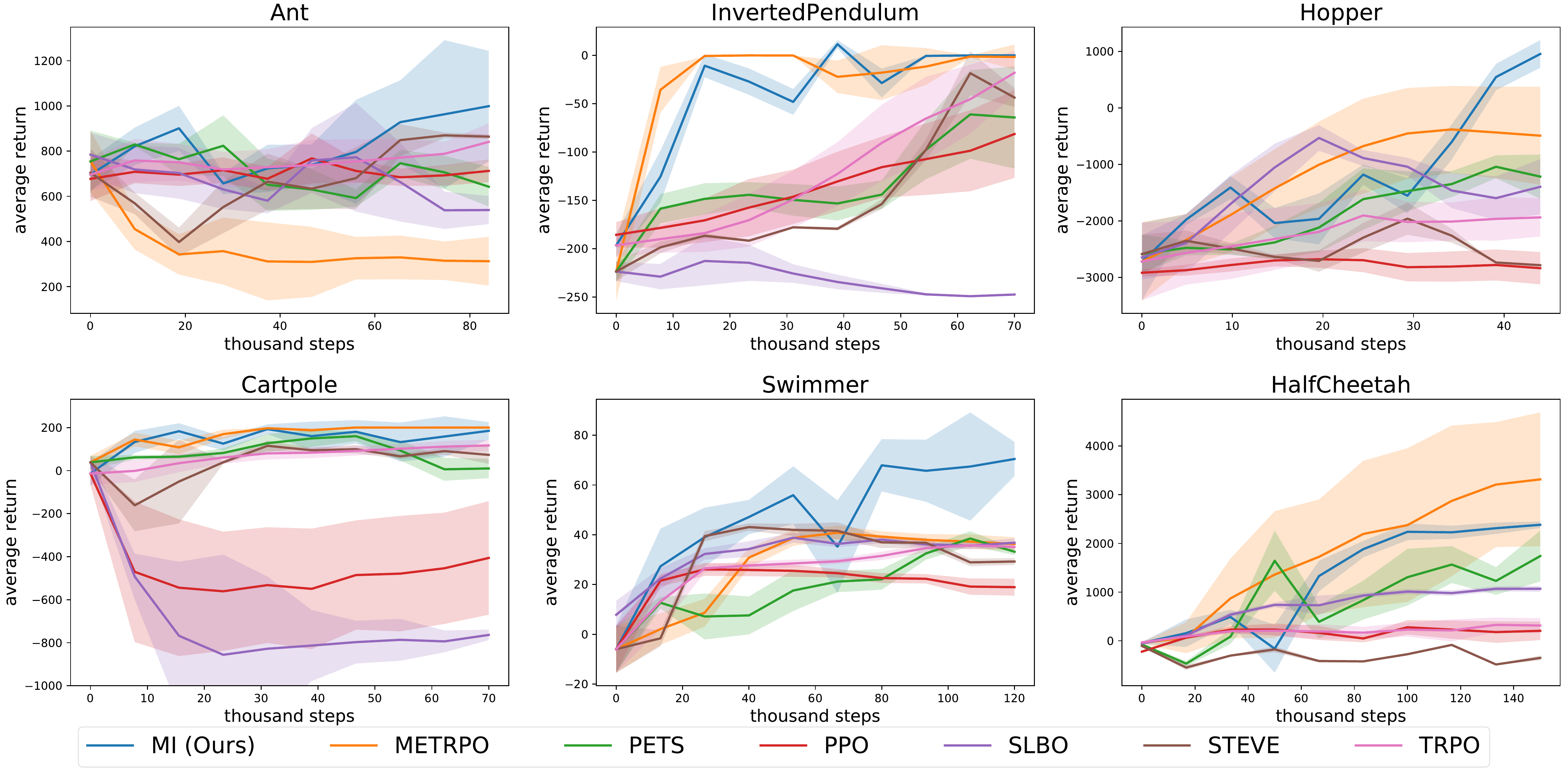}
    \caption{Learning curves of our MI versus two model-free and four model-based baselines. The solid lines indicate the mean of five trials and the shaded regions suggest standard deviation.}
    \label{fig:summary}
\end{figure*}

\begin{table*}[!tbh]
    \vspace{0.5cm}
    \centering
    \caption{Proportion of bench-marked RL methods that are inferior to MI in terms of $5\%$ t-test. $x/y$ indicates that among $y$ approaches, MI is significantly better than $x$ approaches. The detailed performance can be found in Table 1 of \citet{langlois2019benchmarking}. It should be noted that the reported results in \citet{langlois2019benchmarking} are the final performance after 200k time-steps whereas ours are no more than 140k time-steps.}
    \begin{tabular}{c|c|c|c|c|c|c}
    \hline\hline
         & Ant & InvertedPendulum & Hopper & Cartpole & Swimmer & HalfCheetah\\\hline
        MBRL & 8/10 & 8/10 & 8/10 & 8/10 & 9/10 & 8/10 \\
        MFRL & 4/4  & 4/4  & 4/4  & 4/4  & 4/4  & 2/4 \\\hline
    \end{tabular}
    \label{tab:comparison}
\end{table*}

To fairly compare algorithms and enhance reproducibility, we use open-sourced environments released along with a model-based benchmark paper \citep{langlois2019benchmarking}. This prior work is based on a physical simulation engine, MuJoCo \citep{todorov2012mujoco}. Specifically, we evaluate the proposed algorithm MI on six continuous control tasks. Due to the space limit, additional two experiments on Reacher and Pendulum are reported in Appendix~\ref{apdx:extra}. 
For hyper-parameters of Algorithm~\ref{algo:mi} and coefficients such as entropy regularization $\lambda$, please refer to Appendix~\ref{apdx:hyper}.

To assess the benefit of utilizing the proposed MI algorithm, we compare its performance with two model-free algorithms, TRPO \citep{schulman2015trust} and PPO \citep{schulman2017proximal}, and four model-based methods: 
SLBO \citep{luo2018slbo} gave theoretical guarantees of monotonic improvement for model-based deep RL and proposed to update a joint model-policy objective.
PETS \citep{kurtland2018pets} proposed to employ uncertainty-aware dynamic models with sampling-based uncertainty to capture both aleatoric and epistemic uncertainty. 
METRPO \citep{kurutach18metrpo} showed that insufficient data may cause instability and proposed to use an ensemble of models to regularize the learning process. 
STEVE \citep{jacob2018steve} dynamically interpolates among model rollouts of various horizon lengths and favors those with estimates of lower error.

Figure~\ref{fig:summary} shows the learning curves for all methods. In environments other than Ant, MI converges fairly fast and learns a policy that competes or outperforms competitors'. 
In Ant, even though MI does not improve the performance too much from the initial one, the fact that it slightly improve the average return to around 1,000 indicates that MI can capture a better transition than other methods do. 

The questions raised at the beginning of this section can now be answered. The learned model enables TRPO to explore and learn without directly accessing real transitions. As a result, TRPO equipped with MI needs significantly fewer interactions with the real environment to learn a good policy. Even though MI is based on the training framework proposed in SLBO, the additional distribution matching component allows the learned model to generate similar rollouts to that of the real environment. Empirically, this gives superior performance because we rely on long rollouts to estimate the policy gradient.

To better understand the performance presented in Figure~\ref{fig:summary}, we further compare MI with bench-marked RL algorithms recorded in \cite{langlois2019benchmarking}. This includes state-of-the-art MFRL methods such as TD3 \citep{fujimoto2018td3} and SAC \citep{haarnoja2018soft}. The reported results of \cite{langlois2019benchmarking} are the final performance after 200k time-steps, but we only use up to 140k time-steps to train MI. Table~\ref{tab:comparison} indicates that MI significantly outperforms most of the MBRL and MFRL methods with $50\%$ fewer samples. This verifies that by incorporating distribution matching, MI is more sample-efficient.  

\section{Conclusion}
We have pointed out that the state-of-the-art methods concentrate on learning models in a supervised fashion. 
However, this does not guarantee that under the same policy,  
the model is able to reproduce realistic trajectories. In particular, the model may not be accurate enough to synthesize
long rollouts that accurately reflect real rollouts. To address this, we have proposed to use a WGAN to achieve occupancy measure matching between the real 
and the learned transition models. We theoretically show that such matching indicates the closeness in cumulative rewards between the model and the real environment under the same policy.

To enable stable training across WGANs, we suggest a truncated version of WGAN to prevent training from getting stuck at local optima. We have further discovered that in Model Imitation, directly transferring network weights from previous GAN training to the next can be unstable. We attribute this issue to the dissimilarity in occupancy measures, which hampered us from borrowing the concept of GAN adaptation and benefitting from sharing weights. To alleviate this issue, and utilize the knowledge learned from previous models, we proposed matching a mixture of occupancy measures from a collection of policies instead of an occupancy measure from a single policy. Consequently, the discrepancy between two densities can be adjusted and the adaptation by transferring weights is applicable.

The success of GANs in imitation learning motivated us to hypothesize its superiority over supervised learning when learning transition models. We have confirmed it by showing that our MI algorithm possesses sound theoretical properties and is able to converge much faster empirically, obtaining better policies than state-of-the-art model-based and model-free algorithms.

\section*{Acknowledgement}
The authors would like to acknowledge the National Science Founding for the grant RI-1764078 and Qualcomm for the generous support.

\bibliography{reference}

\begin{thebibliography}{42}
\providecommand{\natexlab}[1]{#1}
\providecommand{\url}[1]{\texttt{#1}}
\expandafter\ifx\csname urlstyle\endcsname\relax
  \providecommand{\doi}[1]{doi: #1}\else
  \providecommand{\doi}{doi: \begingroup \urlstyle{rm}\Url}\fi

\bibitem[Arjovsky et~al.(2017)Arjovsky, Chintala, and Bottou]{wgan}
Martin Arjovsky, Soumith Chintala, and Léon Bottou.
\newblock Wasserstein gan, 2017.

\bibitem[Asadi et~al.(2019)Asadi, Misra, Kim, and Littman]{asadi2019combating}
Kavosh Asadi, Dipendra Misra, Seungchan Kim, and Michel~L Littman.
\newblock Combating the compounding-error problem with a multi-step model.
\newblock \emph{arXiv preprint arXiv:1905.13320}, 2019.

\bibitem[Blitzer et~al.(2006)Blitzer, McDonald, and Pereira]{blitzer2006domain}
John Blitzer, Ryan McDonald, and Fernando Pereira.
\newblock Domain adaptation with structural correspondence learning.
\newblock In \emph{Proceedings of the 2006 conference on empirical methods in
  natural language processing}, pages 120--128, 2006.

\bibitem[Buckman et~al.(2018)Buckman, Hafner, Tucker, Brevdo, and
  Lee]{jacob2018steve}
Jacob Buckman, Danijar Hafner, George Tucker, Eugene Brevdo, and Honglak Lee.
\newblock Sample-efficient reinforcement learning with stochastic ensemble
  value expansion.
\newblock In S.~Bengio, H.~Wallach, H.~Larochelle, K.~Grauman, N.~Cesa-Bianchi,
  and R.~Garnett, editors, \emph{NeurIPS}, pages 8224--8234. Curran Associates,
  Inc., 2018.

\bibitem[Canas and Rosasco(2012)]{Canas2012wdist_bound}
Guillermo~D. Canas and Lorenzo~A. Rosasco.
\newblock Learning probability measures with respect to optimal transport
  metrics.
\newblock In \emph{NeurIPS}, pages 2492--2500, 2012.

\bibitem[Chua et~al.(2018)Chua, Calandra, McAllister, and
  Levine]{kurtland2018pets}
Kurtland Chua, Roberto Calandra, Rowan McAllister, and Sergey Levine.
\newblock Deep reinforcement learning in a handful of trials using
  probabilistic dynamics models.
\newblock In S.~Bengio, H.~Wallach, H.~Larochelle, K.~Grauman, N.~Cesa-Bianchi,
  and R.~Garnett, editors, \emph{NeurIPS}, pages 4754--4765. Curran Associates,
  Inc., 2018.

\bibitem[Clavera et~al.(2018)Clavera, Rothfuss, Schulman, Fujita, Asfour, and
  Abbeel]{clavera2018mbmpo}
Ignasi Clavera, Jonas Rothfuss, John Schulman, Yasuhiro Fujita, Tamim Asfour,
  and Pieter Abbeel.
\newblock Model-based reinforcement learning via meta-policy optimization.
\newblock In \emph{CoRL}, pages 617--629, 2018.

\bibitem[Daum{\'e}~III(2009)]{daume2009frustratingly}
Hal Daum{\'e}~III.
\newblock Frustratingly easy domain adaptation.
\newblock \emph{arXiv preprint arXiv:0907.1815}, 2009.

\bibitem[{Deisenroth} et~al.(2015){Deisenroth}, {Fox}, and
  {Rasmussen}]{Deisenroth2015gauss}
M.~P. {Deisenroth}, D.~{Fox}, and C.~E. {Rasmussen}.
\newblock Gaussian processes for data-efficient learning in robotics and
  control.
\newblock \emph{IEEE Transactions on Pattern Analysis and Machine
  Intelligence}, pages 408--423, 2015.

\bibitem[Fujimoto et~al.(2018)Fujimoto, van Hoof, and Meger]{fujimoto2018td3}
Scott Fujimoto, Herke van Hoof, and David Meger.
\newblock Addressing function approximation error in actor-critic methods.
\newblock \emph{arXiv preprint arXiv:1802.09477}, 2018.

\bibitem[Gibbs and Su(2002)]{dist_bounds}
Alison~L. Gibbs and Francis~Edward Su.
\newblock On choosing and bounding probability metrics.
\newblock \emph{International Statistical Review}, pages 419--435, 2002.

\bibitem[Gulrajani et~al.(2017)Gulrajani, Ahmed, Arjovsky, Dumoulin, and
  Courville]{gulrajani2017improved}
Ishaan Gulrajani, Faruk Ahmed, Martin Arjovsky, Vincent Dumoulin, and Aaron~C
  Courville.
\newblock Improved training of wasserstein gans.
\newblock In \emph{NeurIPS}, pages 5767--5777, 2017.

\bibitem[Haarnoja et~al.(2018)Haarnoja, Zhou, Abbeel, and
  Levine]{haarnoja2018soft}
Tuomas Haarnoja, Aurick Zhou, Pieter Abbeel, and Sergey Levine.
\newblock Soft actor-critic: Off-policy maximum entropy deep reinforcement
  learning with a stochastic actor.
\newblock \emph{arXiv preprint arXiv:1801.01290}, 2018.

\bibitem[Ho and Ermon(2016)]{ho2016generative}
Jonathan Ho and Stefano Ermon.
\newblock Generative adversarial imitation learning.
\newblock In \emph{NeurIPS}, pages 4565--4573, 2016.

\bibitem[Janner et~al.(2019)Janner, Fu, Zhang, and Levine]{janner2019trust}
Michael Janner, Justin Fu, Marvin Zhang, and Sergey Levine.
\newblock When to trust your model: Model-based policy optimization.
\newblock \emph{arXiv preprint arXiv:1906.08253}, 2019.

\bibitem[Kamthe and Deisenroth(2017)]{Kamthe2017gauss}
Sanket Kamthe and Marc~Peter Deisenroth.
\newblock Data-efficient reinforcement learning with probabilistic model
  predictive control.
\newblock In \emph{AISTATS}, 2017.

\bibitem[Karras et~al.(2019)Karras, Laine, and Aila]{karras2019style}
Tero Karras, Samuli Laine, and Timo Aila.
\newblock A style-based generator architecture for generative adversarial
  networks.
\newblock In \emph{CVPR}, pages 4401--4410, 2019.

\bibitem[Kupcsik et~al.(2013)Kupcsik, Deisenroth, Peters, and
  Neumann]{Kupcsik2013gauss}
Andras~Gabor Kupcsik, Marc~Peter Deisenroth, Jan Peters, and Gerhard Neumann.
\newblock Data-efficient generalization of robot skills with contextual policy
  search.
\newblock In \emph{AAAI}, pages 1401--1407, 2013.

\bibitem[Kurach et~al.(2019)Kurach, Lu{\v{c}}i{\'c}, Zhai, Michalski, and
  Gelly]{kurach2019large}
Karol Kurach, Mario Lu{\v{c}}i{\'c}, Xiaohua Zhai, Marcin Michalski, and
  Sylvain Gelly.
\newblock A large-scale study on regularization and normalization in gans.
\newblock In \emph{ICML}, pages 3581--3590, 2019.

\bibitem[Kurutach et~al.(2018)Kurutach, Clavera, Duan, Tamar, and
  Abbeel]{kurutach18metrpo}
Thanard Kurutach, Ignasi Clavera, Yan Duan, Aviv Tamar, and Pieter Abbeel.
\newblock Model-ensemble trust-region policy optimization.
\newblock In \emph{ICLR}, 2018.
\newblock URL \url{https://openreview.net/forum?id=SJJinbWRZ}.

\bibitem[Lim and Ye(2017)]{lim2017geometric}
Jae~Hyun Lim and Jong~Chul Ye.
\newblock Geometric gan.
\newblock \emph{arXiv preprint arXiv:1705.02894}, 2017.

\bibitem[Lucic et~al.(2018)Lucic, Kurach, Michalski, Gelly, and
  Bousquet]{lucic2018gans}
Mario Lucic, Karol Kurach, Marcin Michalski, Sylvain Gelly, and Olivier
  Bousquet.
\newblock Are gans created equal? a large-scale study.
\newblock In \emph{NeurIPS}, pages 700--709, 2018.

\bibitem[Luo et~al.(2019)Luo, Xu, Li, Tian, Darrell, and Ma]{luo2018slbo}
Yuping Luo, Huazhe Xu, Yuanzhi Li, Yuandong Tian, Trevor Darrell, and Tengyu
  Ma.
\newblock Algorithmic framework for model-based deep reinforcement learning
  with theoretical guarantees.
\newblock In \emph{ICLR}, 2019.

\bibitem[Mao et~al.(2017)Mao, Li, Xie, Lau, Wang, and
  Paul~Smolley]{mao2017least}
Xudong Mao, Qing Li, Haoran Xie, Raymond~YK Lau, Zhen Wang, and Stephen
  Paul~Smolley.
\newblock Least squares generative adversarial networks.
\newblock In \emph{ICCV}, pages 2794--2802, 2017.

\bibitem[Miyato et~al.(2018)Miyato, Kataoka, Koyama, and
  Yoshida]{miyato2018spectral}
Takeru Miyato, Toshiki Kataoka, Masanori Koyama, and Yuichi Yoshida.
\newblock Spectral normalization for generative adversarial networks.
\newblock \emph{arXiv preprint arXiv:1802.05957}, 2018.

\bibitem[{Nagabandi} et~al.(2018){Nagabandi}, {Kahn}, {Fearing}, and
  {Levine}]{Nagabandi18}
A.~{Nagabandi}, G.~{Kahn}, R.~S. {Fearing}, and S.~{Levine}.
\newblock Neural network dynamics for model-based deep reinforcement learning
  with model-free fine-tuning.
\newblock In \emph{2018 IEEE International Conference on Robotics and
  Automation (ICRA)}, pages 7559--7566, May 2018.
\newblock \doi{10.1109/ICRA.2018.8463189}.

\bibitem[Ross et~al.(2011)Ross, Gordon, and Bagnell]{ross2011dagger}
St{\'e}phane Ross, Geoffrey Gordon, and Drew Bagnell.
\newblock A reduction of imitation learning and structured prediction to
  no-regret online learning.
\newblock In \emph{AISTATS}, pages 627--635, 2011.

\bibitem[Schaal(1997)]{schaal1997learning}
Stefan Schaal.
\newblock Learning from demonstration.
\newblock In \emph{NeurIPS}, pages 1040--1046, 1997.

\bibitem[Schulman et~al.(2015)Schulman, Levine, Abbeel, Jordan, and
  Moritz]{schulman2015trust}
John Schulman, Sergey Levine, Pieter Abbeel, Michael Jordan, and Philipp
  Moritz.
\newblock Trust region policy optimization.
\newblock In \emph{ICML}, pages 1889--1897, 2015.

\bibitem[Schulman et~al.(2017)Schulman, Wolski, Dhariwal, Radford, and
  Klimov]{schulman2017proximal}
John Schulman, Filip Wolski, Prafulla Dhariwal, Alec Radford, and Oleg Klimov.
\newblock Proximal policy optimization algorithms.
\newblock \emph{arXiv preprint arXiv:1707.06347}, 2017.

\bibitem[Sutton and Barto(1998)]{sutton1998introduction}
Richard~S Sutton and Andrew~G Barto.
\newblock \emph{Introduction to Reinforcement Learning}, volume 135.
\newblock MIT press, 1998.

\bibitem[Syed and Schapire(2010)]{syed2010}
Umar Syed and Robert~E Schapire.
\newblock A reduction from apprenticeship learning to classification.
\newblock In J.~D. Lafferty, C.~K.~I. Williams, J.~Shawe-Taylor, R.~S. Zemel,
  and A.~Culotta, editors, \emph{NeurIPS}, pages 2253--2261. Curran Associates,
  Inc., 2010.

\bibitem[Syed et~al.(2008)Syed, Bowling, and Schapire]{syed2008apprenticeship}
Umar Syed, Michael Bowling, and Robert~E Schapire.
\newblock Apprenticeship learning using linear programming.
\newblock In \emph{ICML}, pages 1032--1039, 2008.

\bibitem[Talvitie(2017)]{talvitie2017self}
Erik Talvitie.
\newblock Self-correcting models for model-based reinforcement learning.
\newblock In \emph{AAAI}, 2017.

\bibitem[Todorov et~al.(2012)Todorov, Erez, and Tassa]{todorov2012mujoco}
Emanuel Todorov, Tom Erez, and Yuval Tassa.
\newblock Mujoco: A physics engine for model-based control.
\newblock In \emph{IROS}, pages 5026--5033, 2012.

\bibitem[Villani(2008)]{Villani2008_opt_transport}
C~Villani.
\newblock \emph{Optimal transport -- Old and new}, volume 338, pages xxii+973.
\newblock 01 2008.

\bibitem[Wang et~al.(2019{\natexlab{a}})Wang, Bao, Clavera, Hoang, Wen,
  Langlois, Zhang, Zhang, Abbeel, and Ba]{langlois2019benchmarking}
Tingwu Wang, Xuchan Bao, Ignasi Clavera, Jerrick Hoang, Yeming Wen, Eric
  Langlois, Shunshi Zhang, Guodong Zhang, Pieter Abbeel, and Jimmy Ba.
\newblock Benchmarking model-based reinforcement learning, 2019{\natexlab{a}}.

\bibitem[Wang et~al.(2018)Wang, Wu, Herranz, van~de Weijer, Gonzalez-Garcia,
  and Raducanu]{wang2018transferring}
Yaxing Wang, Chenshen Wu, Luis Herranz, Joost van~de Weijer, Abel
  Gonzalez-Garcia, and Bogdan Raducanu.
\newblock Transferring gans: generating images from limited data.
\newblock In \emph{ECCV}, pages 218--234, 2018.

\bibitem[Wang et~al.(2019{\natexlab{b}})Wang, Gonzalez-Garcia, Berga, Herranz,
  Khan, and van~de Weijer]{wang2019minegan}
Yaxing Wang, Abel Gonzalez-Garcia, David Berga, Luis Herranz, Fahad~Shahbaz
  Khan, and Joost van~de Weijer.
\newblock Minegan: effective knowledge transfer from gans to target domains
  with few images.
\newblock \emph{arXiv preprint arXiv:1912.05270}, 2019{\natexlab{b}}.

\bibitem[Wu et~al.(2018)Wu, Herranz, Liu, van~de Weijer, Raducanu,
  et~al.]{wu2018memory}
Chenshen Wu, Luis Herranz, Xialei Liu, Joost van~de Weijer, Bogdan Raducanu,
  et~al.
\newblock Memory replay gans: Learning to generate new categories without
  forgetting.
\newblock In \emph{NeurIPS}, pages 5962--5972, 2018.

\bibitem[Yu et~al.(2017)Yu, Zhang, Wang, and Yu]{yu2017seqgan}
Lantao Yu, Weinan Zhang, Jun Wang, and Yong Yu.
\newblock Seqgan: Sequence generative adversarial nets with policy gradient.
\newblock In \emph{AAAI}, 2017.

\bibitem[Zhang et~al.(2018)Zhang, Goodfellow, Metaxas, and
  Odena]{zhang2018self}
Han Zhang, Ian Goodfellow, Dimitris Metaxas, and Augustus Odena.
\newblock Self-attention generative adversarial networks.
\newblock \emph{arXiv preprint arXiv:1805.08318}, 2018.

\end{thebibliography}
\bibliographystyle{plainnat}


\newpage
\appendix
\allowdisplaybreaks

\section{Proofs}\label{apdx:proofs}

\begin{lemma}
	Let $\rho_T(s,a)$ be a the normalized occupancy measure generated by the triple $(\alpha,\pi,T)$ with discount factor $\gamma$. Let $\rho_T^\beta(s,a)$ be the normalized occupancy measure generated by the triple $(\rho_T,\pi,T)$ with discount factor $\beta$. If $\gamma>\beta$, then $D_{TV}(\rho_T\lvert\rvert \rho_T^\beta)\leq (1-\gamma)\beta/(\gamma-\beta)$.
	\label{lemma:short_occu_bound}
\end{lemma}
\begin{proof}
	By definition of the occupancy measure we have
	\begin{equation*}
	\begin{split}
	&\rho_T(s,a)=\sum\limits_{i=0}^\infty (1-\gamma)\gamma^i f_i(s,a).\\
	&\rho_T^\beta(s,a)=\sum\limits_{i=0}^\infty \sum\limits_{j=0}^i (1-\gamma)\gamma^{i-j}(1-\beta)\beta^j f_i(s,a),
	\end{split}
	\end{equation*}
	where $f_i(s,a)$ is the density of $(s,a)$ at time $i$ if generated by the triple $(\alpha,\pi,T)$. The TV distance is bounded by
	\begin{equation*}
	\begin{split}
	D_{TV}(\rho_T\lvert\rvert \rho_T^\beta )&\leq \frac{1}{2}\sum\limits_{i=0}^\infty\Big|(1-\gamma)\gamma^i - \sum\limits_{j=0}^i(1-\gamma)\gamma^{i-j}(1-\beta)\beta^j \Big|=\frac{1}{2}\sum\limits_{i=0}^\infty (1-\gamma)\gamma^i\Big| 1-\sum\limits_{j=0}^i (1-\beta)\Big(\frac{\beta}{\gamma}\Big)^j \Big|\\
	&=\frac{1}{2}\sum\limits_{i=0}^\infty (1-\gamma)\gamma^i\frac{1}{\gamma-\beta}\Big|-\beta(1-\gamma) + \Big(\frac{\beta}{\gamma}\Big)^{i+1}(1-\beta)\gamma  \Big|\\
	&\overset{(*)}{=} \frac{(1-\gamma)\beta}{\gamma-\beta}\sum\limits_{i=0}^{M-1} -(1-\gamma)\gamma^i + (1-\beta)\beta^i=\frac{(1-\gamma)\beta}{\gamma-\beta}(\gamma^{M}-\beta^{M})\\
	&\leq \frac{(1-\gamma)\beta}{\gamma-\beta}.
	\end{split}
	\end{equation*}
	where $(*)$ comes from that $-\beta(1-\gamma)+(\frac{\beta}{\gamma})^i(1-\beta)\gamma$ is a strictly decreasing function in $i$. Since $\gamma>\beta$, its sign flips from $+$ to $-$ at some index; say $M$. Finally, the sum of the absolute value are the same from $\sum_{i=0}^{M-1}$ and from $\sum_{i=M}^\infty$ because the total probability is conservative, and the difference on one side is the same as that on the other.
\end{proof}

\section{Experiments}
\subsection{Implementation Details}\label{apdx:details}
We normalized the collected states to have zero mean and unit standard deviation.
The mean $\mu_\mathrm{n}$ and the standard deviation $\sigma_\mathrm{n}$ were updated after real transition tuples $(s_t, a_t, s_{t+1})$ are collected  each time. Instead of directly predicting the next state, we estimated the state difference $s_{t+1}-s_t$ \citep{kurutach18metrpo,luo2018slbo}. Since we incorporated state normalization, the transition network was trained to predict $(s_{t+1}-s_t)/\sigma_n$.

In order to enforce the Lipschitz constraint to the WGAN critic $f$, we employed gradient penalty \citep{gulrajani2017improved} with its weight $\lambda=10$.
As prior studies \cite{kurach2019large,lucic2018gans} found, the adversarial training can be largely benefit from spectral normalization \citep{miyato2018spectral} in our experiments. We incorporated spectral normalization to derive transition functions that showed better stability when they were leveraged to collect samples to update the agent policy.

As mentioned in Section~\ref{sec:empirical}, we clipped the standard deviation of the agent policy from below to maintain a certain degree of exploration, which was found to be beneficial to the transition learning. In addition, to prevent the agent policy from overfitting the transition by increasing standard deviation too much, we also clipped the value from above with $\sigma_\mathrm{max}$. We fixed the values with $\sigma_\mathrm{max}=0.3$ and $\sigma_\mathrm{min}=0.1$ across all experiments.

\subsection{Additional Experiments}\label{apdx:extra}
\begin{figure*}[!ht]
    \centering
    \includegraphics[scale=0.3]{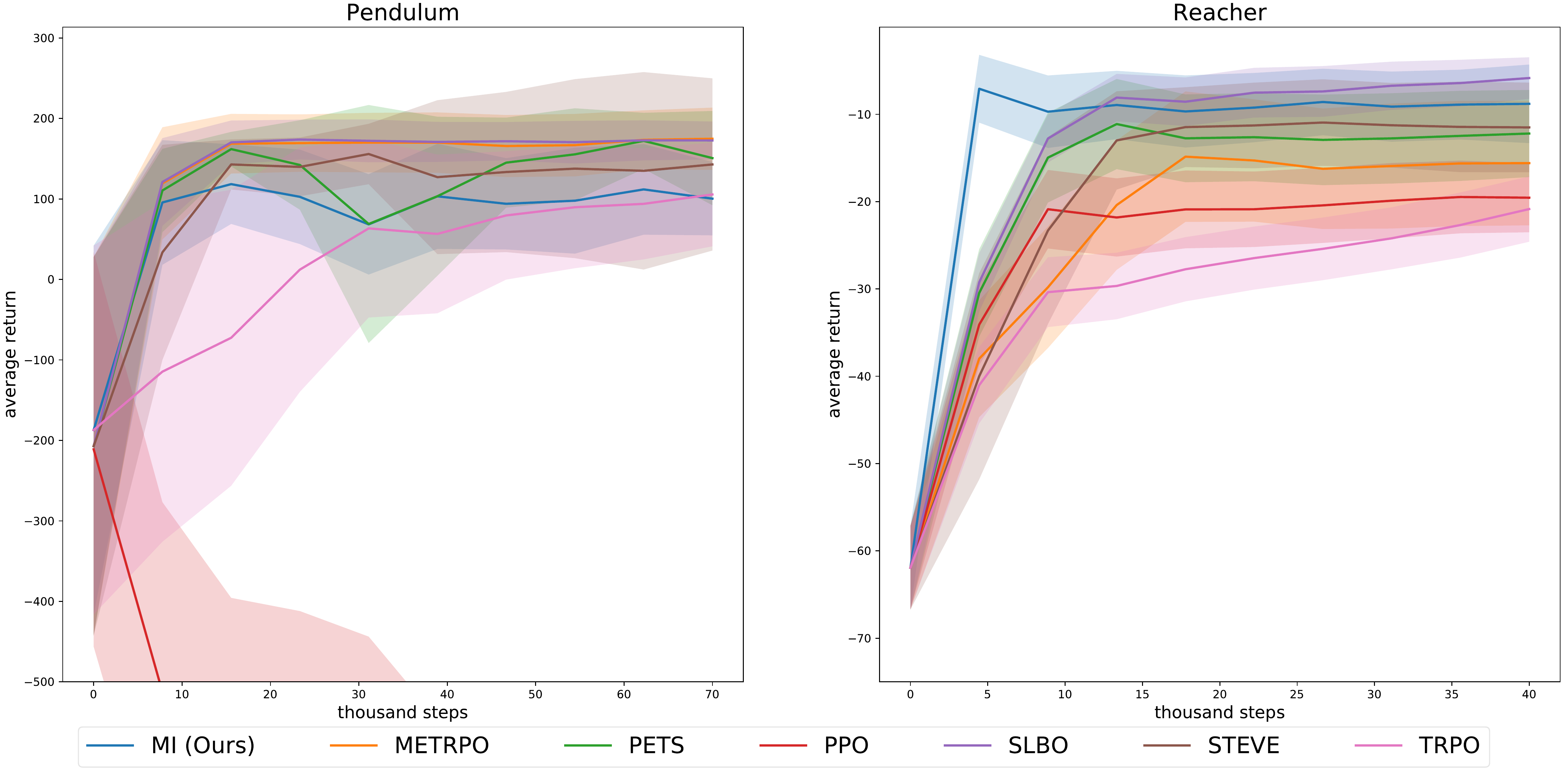}
    \caption{Learning curves of our MI versus two model-free and four model-based baselines. The solid lines indicate the mean of five trials and the shaded regions suggest standard deviation.}
    \label{fig:more_exp}
\end{figure*}

\subsection{Hyperparameters}\label{apdx:hyper}
\begin{table}[!ht]
    \centering
    \resizebox{\columnwidth}{!}{%
    \begin{tabular}{|c|c|c|c|c|c|c|c|c|}
        \hline
         & Ant & InvertedPendulum & Hopper & Cartpole & Swimmer & HalfCheetah & Pendulum & Reacher
        \\\hline
        $N$ & 5 & 1 & 10 & 1 & 2 & 2 & 5 & 1 \\\hline
        $\eta$ & 15 & 1 & 10 & 5 & 10 & 10 & 15 & 5\\\hline
        $n_\text{transition}$ & 200 & 100 & 100 & 100 & 50 & 200 & 100 & 100\\\hline
        $n_\text{policy}$ & 20 & 50 & 10 & 50 & 20 & 20 & 10 & 50\\\hline
        horizon for 
        model update & 20 & 15 & 10 & 15 & 15 & 20 & 20 & 15\\\hline
        entropy
        regularization & $10^{-5}$ & $10^{-5}$ & $10^{-3}$ & $10^{-5}$ & $10^{-7}$ & $10^{-7}$ & $10^{-5}$ & $10^{-5}$\\\hline
    \end{tabular}%
    }
    \caption{List of hyper-parameters adopted in our experiments.}
    \label{tab:hyper}
\end{table}

\end{document}